%% file: root.tex
\documentclass[letterpaper, 10 pt, conference]{IEEEtran}
\IEEEoverridecommandlockouts                             

\usepackage{graphics}

\usepackage{amsmath}
\usepackage{amsthm}
\usepackage{amssymb}
\usepackage{acro}
\input{include/acronyms.tex}
\usepackage{siunitx}
\usepackage{mathtools}
\usepackage{booktabs}
\usepackage[algo2e]{algorithm2e}
\usepackage{subcaption}
\usepackage[bottom=57pt,top=54pt, left=54pt, right=54pt]{geometry}   %57-43

\usepackage{adjustbox}
\usepackage{pbox}
\usepackage{dblfloatfix}
\usepackage{graphicx}
\usepackage[dvipsnames]{xcolor}
\usepackage{hyperref}
\usepackage{caption}
\usepackage[nospace]{cite}
\usepackage{bm}
\usepackage{algorithm}
\usepackage{algcompatible}

\usepackage{array,multirow,graphicx}
\usepackage{makecell}
\usepackage{pifont}
\usepackage{xparse}
\usepackage{enumitem}

\newcommand{\etal}{\textit{et al.}\xspace}
\newcommand{\eg}{\textit{e.g.}\xspace}
\newcommand{\ie}{\textit{i.e.},\xspace}

\input{include/newcommands.tex}

\newtheorem{assumption}{Assumption}
\newtheorem{proposition}{Proposition}[section]
\newtheorem{definition}{Definition}[section]
\newtheorem{corollary}{Corollary}[section]

\NewDocumentCommand{\todo}{o m}{\textcolor{red}{\textbf{TODO\IfNoValueTF{#1}{}{(#1)}:} #2}}
\NewDocumentCommand{\note}{o m}{\textcolor{orange}{\textbf{NOTE\IfNoValueTF{#1}{}{(#1)}:} #2}}
\newlist{todolist}{itemize}{2} \setlist[todolist]{label=$\square$}

\newcommand{\tfirst}{$\bar{\tau}_1$\xspace}
\newcommand{\topt}{$\bar{\tau}_\star$\xspace}
\newcommand{\tterm}{$\bar{\tau}_\infty$\xspace}

\newcommand{\argmin}[1]{\underset{#1}{\text{argmin}}\xspace}
\newcommand{\argmax}[1]{\underset{#1}{\text{argmax}}\xspace}
\newcommand{\parsec}[1]{\noindent\textbf{#1.}\xspace}

\begin{document}

\title{Traversing Mars: Cooperative Informative Path Planning to Efficiently Navigate Unknown Scenes}

%\markboth{IEEE Robotics and Automation Letters. Preprint Version. Accepted February, 2021}

\author{Friedrich M. Rockenbauer$^1$, Jaeyoung Lim$^1$,  Marcus G. Müller$^{1,2}$, Roland Siegwart$^1$, and Lukas Schmid$^3$% <-this % stops a space

\thanks{$^1$Autonomous Systems Lab, ETH Z\"urich, Switzerland. {\tt \footnotesize friedrich.rockenbauer@ethz-asl.ch,  \{jalim, rsiegwart\}@ethz.ch}}%
\thanks{$^{2}$Institute of Robotics and Mechatronics, German Aerospace Center (DLR), Germany. {\tt \footnotesize marcus.mueller@dlr.de}}
\thanks{$^3$MIT SPARK Lab, Massachusetts Institute of Technology, USA. {\tt \footnotesize lschmid@mit.edu}}%
\thanks{This work was supported by ETH Research Grant AvalMapper ETH-10 20-1 and the Swiss National Science Foundation (SNSF) grant No. 214489.}%
}%

\maketitle

%%%%%%%%%%%%%%%%%%%%%%%%%%%%%%%%%%%%%%%%%%%%%%%%%%%%%%%%%%%%%%%%%%%%%%%%%%%%%%%%
\begin{abstract}

The ability to traverse an unknown environment is crucial for autonomous robot operations.
However, due to the limited sensing capabilities and system constraints, approaching this problem with a single robot agent can be slow, costly, and unsafe. 
For example, in planetary exploration missions, the wear on the wheels of a rover from abrasive terrain should be minimized at all costs as reparations are infeasible.
On the other hand, utilizing a scouting robot such as a micro aerial vehicle (MAV) has the potential to reduce wear and time costs and increasing safety of a follower robot.
This work proposes a novel cooperative \ac{IPP} framework that allows a scout (e.g., an MAV) to efficiently explore the minimum-cost-path for a follower (e.g., a rover) to reach the goal.
We derive theoretic guarantees for our algorithm, and prove that the algorithm always terminates, always finds the optimal path if it exists, and terminates early when the found path is shown to be optimal or infeasible.
We show in thorough experimental evaluation that the guarantees hold in practice, and that our algorithm is 22.5\% quicker to find the optimal path and 15\% quicker to terminate compared to existing methods.

\end{abstract}

%\begin{IEEEkeywords}
%Planetary Exploration, Path Planning, Reactive Planning; Aerial Systems, Perception and Autonomy
%\end{IEEEkeywords}
% ===============================================================================================

% \section*{Additional Material}
% \begin{itemize}
%     \item Code: \href{https://github.com/ethz-asl/asldoc-2022-MA-Friedrich-Planetary-Exploration}{https://github.com/ethz-asl/asldoc-2022-MA-Friedrich-Planetary-Exploration}
%     \item Video: 
% \end{itemize}
% ===============================================================================================

\section{Introduction}
\label{sec:introduction}
Navigating through an unknown environment is an essential capability for robot autonomy. Robots need to be able to explore the environment and plan a safe and efficient path~\cite{oleynikovaSafeLocalExploration2018}.
However, due to the limited information available to the robot, exploring the unknown environment while ensuring efficiency is challenging.
This fact is further exacerbated if the robot's motion incurs a notable cost, such as wear on parts of the robot, that should be minimized.

Recently, heterogeneous collaborative robotic teams have gained interest as a more robust and capable alternative to deploying a single robot, by taking advantage of the capabilities of each agent~\cite{RobotTeamsSwarms2019}. 
For navigation, if the travel cost is asymmetric between different robots, having a robot scout for the other, which we denote a follower robot, can be advantageous. 
Therefore, the combination of aerial and ground robots has been used in various applications~\cite{marconi2012sherpa, fankhauser2016collaborative, delmericoActiveAutonomousAerial2017, youngRotorcraftMarsScouts2002, dlr136354} to take advantage of the agile characteristic of an aerial scout to explore unknown environments~\cite{delmericoActiveAutonomousAerial2017, folsomScalableInformationtheoreticPath2021, fankhauser2016collaborative}. 
The explored environment can then be used to plan a path for the follower robot to transport heavy payloads or operate precise instruments. 
However, naively exploring the scene with an aerial scout could require extensive scene coverage.
For example, achieving high coverage of the environment may not always contribute to discovering a better rover path~\cite{fankhauser2016collaborative}. 
Since flight time is a major limitation of aerial robots, minimizing the exploration required by the scout is essential for longer-range missions.
In addition, existing methods do not explicitly consider the traversal cost of the follower robot and often rely on heuristics such as frontiers, whereas the geometrically shortest path may be sub-optimal in terms of accrued cost.

\begin{figure}[t]
    \centering
    \includegraphics[width=\linewidth]{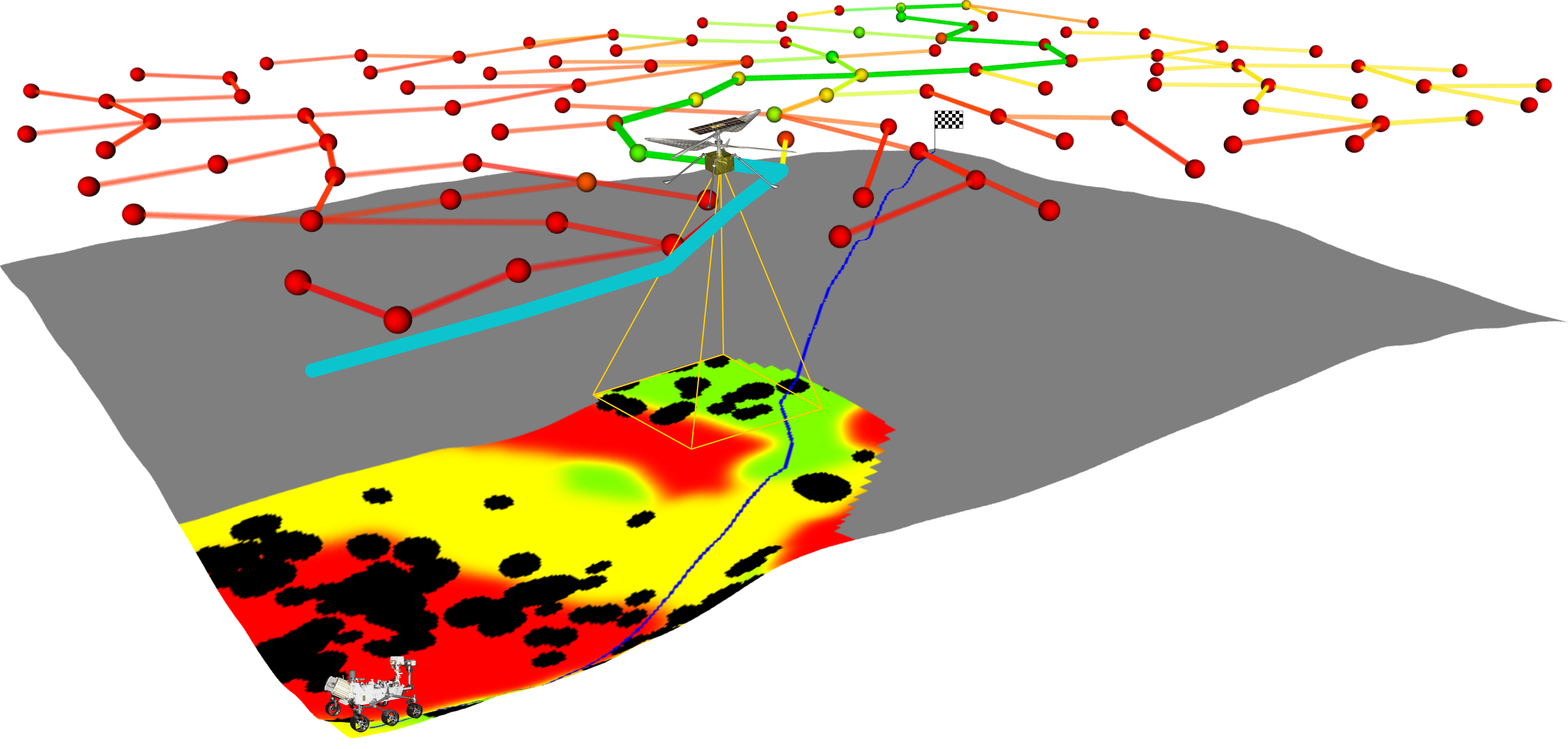}
    \caption{Illustration of our problem and approach. A follower robot (\eg a rover, bottom) has to traverse through unknown space to reach a goal (top) while minimizing traversal cost (soil color, from red=high to green=low). The goal is for a scouting robot (\eg an MAV) to explore the scene such that the follower path is optimal. 
    Our approach generates an optimistic follower path (blue) guiding the scouting-IPP. For illustration, sparsified IPP-samples are shown colored by information gain (colored from red=low to green=high), with the past (teal) and currently planned path (green) of the scout.}
    \label{fig:paper_overview}
    \vspace{-10pt}
\end{figure}

% Previous works use heuristics in the exploration reward, such as weighting frontiers~\cite{zhangFastActiveAerial2022} or view utilities~\cite{delmericoActiveAutonomousAerial2017}. However, these heuristics can be sensitive to the environment and do not provide any formal guarantee on either the performance or termination criteria of the exploration.

In this work, we develop a novel \ac{IPP} approach to guide the exploration of a scout robot to find the optimal traversable path of a follower robot in unknown scenes. 
We propose a novel utility function to guide the exploration of the aerial scout to cover the lowest-cost traversable path for the follower robot. 
We derive theoretical guarantees for our algorithm, showing that our approach always finds the optimally traversable path if it exists.
We further propose two termination criteria that allow early termination of exploration by determining that the current path is in fact the optimal path, or that no feasible path can exist.

While the developed methodology is general, we evaluate our method in an extra-terrestrial exploration scenario, where the cost represents wear on the follower rover wheels as a function of the traversed terrain \cite{swanAI4MARSDatasetTerrainAware2021}.
The aerial scout is tasked to discover the optimal traversable path for the follower, where optimal refers to minimizing the rover wear, which can significantly improve safety and extend the lifetime of the rover. 
We show in thorough experimental evaluations that the derived guarantees hold in practice, and that our proposed method significantly outperforms existing approaches in terms of scouting time required and quality of the discovered follower path.
We make the following contributions: 
\begin{itemize}
    \item We propose a novel follower-path-aware \ac{IPP} formulation for an agile scout, minimizing the traversal cost of a follower robot in unknown scenes.
    \item We formally prove the optimality and completeness of the proposed algorithm and derive termination criteria once the path is shown to be optimal or no path exists.
    \item We show in thorough experimental evaluation that the guarantees hold in practice, and that our algorithm finds the optimal path  $22.5\%$ quicker and terminates $15\%$ faster compared to existing methods.
    We release our implementation as open-source\footnote{Released upon acceptance at \url{https://github.com/ethz-asl/scouting-ipp}.}. 
\end{itemize}
% ===============================================================================================
\section{Related Work} 
\label{sec:rel_work}
\subsection{Informative Path Planning} 
\label{sec:rel_work_ipp}

Using \ac{IPP} methods to explore an unknown scene has proven useful for various applications from exploration~\cite{oleynikovaSafeLocalExploration2018} to inspection~\cite{bircherThreedimensionalCoveragePath2016} and information gathering~\cite{moon2022tigris}. 
Frontier-based methods~\cite{yamauchiFrontierbasedApproachAutonomous1997, dengFrontierbasedAutomaticdifferentiableInformation2020, dengRoboticExplorationUnknown2020} use intermediate goals along the boundaries of known and unknown space.
While this guarantees complete exploration, it may lead to sub-optimal exploration speed.
Alternatively, sampling-based methods find incremental next-best-views to explore the unstructured environment~\cite{hollingerSamplingbasedRoboticInformation2014, Schmid20ActivePlanning, bircherRecedingHorizonNextBestView2016, oleynikovaSafeLocalExploration2018, moon2022tigris}. 
In addition, using the utility function of~\iac{IPP} methods to incorporate downstream task objectives has been successful in various applications such as reconstruction~\cite{Schmid20ActivePlanning, bircherThreedimensionalCoveragePath2016}, navigation~\cite{oleynikovaSafeLocalExploration2018, dengRoboticExplorationUnknown2020}, or semantic scene understanding\cite{zurbrugg2022embodied, papatheodorou2023finding}. 
In this work, we also leverage this flexibility to propose a novel follower-path-aware utility function to guide the exploration of the aerial scout.

In recent years, learning-based methods have made a notable impact on many aspects of exploration planning.
Most approaches can be grouped into learning to improve parts of a classical planner~\cite{hepp2018learn, schmid2022fast}, learning environmental features to improve exploration~\cite{schmid2022sc, tao2023seer}, or learning the next best action end-to-end, \eg using reinforcement learning~\cite{garaffa2021reinforcement, cao2023ariadne}.
While these approaches are applicable to our problem, we focus on a classical planning approach in this work to derive thorough performance guarantees.

\subsection{Collaborative Robot Teams} 
\label{sec:rel_work_team}
Collaborative robotic teams have found impact in various applications such as mapping~\cite{Michael2014, drones4040079}, localization~\cite{9381949, schmuck2019, sasakiWhereMapIterative2020}, navigation~\cite{fankhauser2016collaborative, kaslinCollaborativeLocalizationAerial2016} and exploration~\cite{kaslinCollaborativeLocalizationAerial2016}. 
Especially aerial-ground robotic teams have gained interest as both agents can take advantage of their different locomotion modalities~\cite{fankhauser2016collaborative, kaslinCollaborativeLocalizationAerial2016, dangGraphbasedSubterraneanExploration2020, tranzattoCERBERUSAutonomousLegged2021}.
\cite{kaslinCollaborativeLocalizationAerial2016}. Dang~\etal\cite{dangGraphbasedSubterraneanExploration2020} and Tranzatto~\etal~\cite{tranzattoCERBERUSAutonomousLegged2021} propose a graph-based planner to first explore the environment with an aerial scout and use the resulting map to navigate a follower ground robot.
Most closely related to us, Delmerico~\etal~\cite{delmericoActiveAutonomousAerial2017} and Zhang~\etal~\cite{zhangFastActiveAerial2022} both optimize the total response time of the team and weigh frontiers by a combination of expected scout exploration time and follower time-to-go.
However, both methods are based on frontiers and can thus be sensitive to the distribution of environments and may perform poorly in unstructured environments. 
In addition, we consider a general follower cost that can \eg be a function of the traversed space.
Similarly, Folsom~\etal~\cite{folsomScalableInformationtheoreticPath2021} use an aerial scout to discover the traversibility cost map of a rover.
In contrast to our problem, they assume a prior map of the cost field including uncertainties is given, and plan a path to minimize the uncertainty of the follower path. 

Finally, previous works considering aerial-ground robot teams do not consider termination criteria that provide performance guarantees. 
In this work, we propose criteria for terminating the exploration for the aerial scout and prove that the proposed algorithm is complete and the discovered solution is optimal if one exists.

% ===============================================================================================
\section{Problem Statement}
\label{sec:problem}
This work addresses the two-body collaborative \emph{scout and follower} planning problem in unknown, unstructured environments.
We consider a heterogeneous robotic team composed of a \emph{scout} $S$, e.g. an aerial drone, and a \emph{follower} $F$, e.g. a ground rover.
Both are operating in a continuous space $\mathcal{X}$, divided into traversable space $\mathcal{S}^+$ and intraversable space $\mathcal{S}^- =\mathcal{X} \setminus \mathcal{S}^+$. 
The \emph{scout} is assumed to be more maneuverable than the \emph{follower}, thus the traversable space of the \emph{follower} $\mathcal{S}_F^+$ is assumed to be a subset of that of the \emph{scout} $\mathcal{S}^+_F \subseteq \mathcal{S}^+_S$. 

Each element in the permitted space has a bounded cost function for the scout $C_S: \mathcal{S}^+_S \mapsto [c_S^{min}, c_S^{max}] \subset \mathbb{R}^+$ and follower $C_F: \mathcal{S}^+_F \mapsto [c_F^{min}, c_F^{max}]\subset \mathbb{R}^+$. 
The goal is to find the optimal path for the \emph{follower} $P_F^\star: [0, T_F] \mapsto \mathcal{X}$ from $P_F(0) = \bm{x}_{start}$ to $P_F(T_F) = \bm{x}_{goal}$, where $T_F$ is the time to complete the path\footnote{Note that for a constant velocity model this can easily be converted to optimize over path lengths instead.}.
We denote the accumulated traversability cost of a path as $Q_F(P_F, C_F) = \int_{P_F} C_F(\bm{x})d\bm{x}$.

The optimal follower path $P_F^\star$ can thus be defined as:
\begin{align} \label{eq:optimal_rover_cost}
    P_F^\star = &\ \argmin{P_F} \quad Q_F(P_F, C_F) \\
    s.t. \quad & P_F(t) \in \mathcal{S}^+_F & \forall t \in [0, T_F] \nonumber\\ 
    & P_F(0) = \bm{x}_{start},\quad P_F(T_F) = \bm{x}_{goal}\nonumber
\end{align}

We assume the cost field $C_F$ and traversable space $S_F^+$ is originally unknown, and the scout has to explore the environment to establish a safe and optimal follower path. 
The scout incrementally creates a partial cost field $M: \mathcal{S}^E \mapsto C_F$, where $\mathcal{S}^E$ denotes the explored space.

Thus, we can state the \emph{scouting IPP} problem as~\eqref{eq:problem_scout}, where we find the minimal cost path of the scout to discover the optimal path of the follower:

\begin{align}\label{eq:problem_scout}
    P_S^\star =\ & \argmin{P_S} \int_{P_S} C_S(\bm{x})d\bm{x}\\
    s.t. \quad &P_S(0) = \bm{x}_{start}\nonumber\\
    & P_S(t) \in \mathcal{S}^+_S & \forall t \in [0, T_S]\nonumber\\ %
    & Q_F(P_F^{\star}, M)= Q_F(P_F^\star, C_F)\nonumber
\end{align}
The last condition notes that the optimal cost of the follower path based on the partial map $M$ should be the same as the optimal cost of the full cost field as defined in~\eqref{eq:optimal_rover_cost}.

We note that this \ac{IPP} formulation assumes a sufficient cost-budget of the scout such that the optimal follower path can always be found.
Nonetheless, we show in Sec.~\ref{sec:evaluation} that this objective also results in favorable performance when the scout is subject to an insufficient budget constraint.
% ===============================================================================================
\section{Follower Path Planning}
\label{sec:follower_path_planning}
\subsection{Preliminaries}
\label{sec:planning_for_follower}
The proposed method can be paired with any follower planner that satisfies the following assumption.

\begin{assumption}[Follower Planner]
    The follower planner is complete and optimal.
    \label{ass:rover_planner}
\end{assumption}

In this work, we employ A$^\star$ search over all 8-connected cells of a 2D map $M$, which can be shown to satisfy assumption~\ref{ass:rover_planner}~\cite{russell2010artificial}.
We note that many popular sampling-based planners such as RRT* and PRM*~\cite{karamanSamplingbasedAlgorithmsOptimal2011} asymptotically satisfy assumption~\ref{ass:rover_planner}. 
While both optimality and completeness can not be strictly guaranteed within finite time, we empirically find that the assumption holds given sufficient time for the planner to converge.

% ----------------------------------------------------------------------------------------------

Since we assume the planning is done on a partial map constructed by the \emph{scout}, we introduce \emph{feasible} and \emph{optimistic} paths that can be found on a partial map.

\begin{definition}[Feasible Path]\label{def:feasible_path}
    A feasible path $\Bar{P}_F$ in map $M$ is the optimal path \eqref{eq:optimal_rover_cost} defined in the feasible set $\mathcal{S}^E \cap \mathcal{S}_F^+$.
    \begin{align}
        \Bar{P}_F = &\argmin{P_F} \quad Q_F(P_F, M)&\\
        &s.t.\quad P_F(t) \in \mathcal{S}^E \cap \mathcal{S}_F^+&\forall t\in [0, T_F]\nonumber
    \end{align}
\end{definition}

The feasible path is useful for determining the feasibility of the problem as well as the global optimality of the found path. 
Intuitively, the added constraint of the feasible path requires it to lie fully in the explored space, thus guaranteeing that the path is traversable.

\begin{definition}[Optimistic Map and Optimistic Path]
    We define the optimistic map $\widetilde{M}$ as the map where all unknown space is considered traversable and of a constant cost $c$.
    The optimistic path $\widetilde{P}$ is the optimal path \eqref{eq:optimal_rover_cost} in the optimistic map $\widetilde{M}$.
    \begin{align}
        \widetilde{P}_F = &\argmin{P_F} \quad Q_F(P_F, \widetilde{M})\\
        &s.t.\quad P_F(t) \in (\mathcal{S}^E \cap \mathcal{S}_F^+)\cup (\mathcal{X} \setminus \mathcal{S}^E)&\forall t\in [0, T_F]\nonumber
    \end{align}
\end{definition}

In contrast to feasible paths, optimistic paths can be generated as candidate paths to guide the scouting planner.

% ---------------------------------------------------------------------------------------------------

\subsection{Feasibility}
In this section, we look into the conditions on whether there exists a true feasible \emph{follower} path based on the information in the partial map.

\begin{proposition}[Feasible Problem] \label{prop:problem_feasible}
    Once a feasible path is found, the problem is feasible.
    \begin{align}
        \exists \bar{P}_F \rightarrow \exists P^\star_F \nonumber
    \end{align}
\end{proposition}
\begin{proof}
    Let us define at a given time $k$ an explored space $\mathcal{S}^E_k$. 
    Further exploration can only make the explored space larger $\mathcal{S}^E_k \subseteq \mathcal{S}^E_{k+i} \forall i > 0$.
    Therefore, the path $\bar{P}_F$ must remain feasible and is by definition \ref{def:feasible_path} a valid solution to \eqref{eq:optimal_rover_cost}.
\end{proof}

Initially, no feasible path exists as the environment is completely unknown. 
We thus instead compute an optimistic path $\widetilde{P}_F$ to guide the scouting-IPP. 

\begin{proposition}[Infeasible Problem] \label{prop:termination_infeasible}
    If an optimistic path does not exist, the problem is infeasible.
    \begin{align}
        \not\exists \widetilde{P}_F \rightarrow \not\exists P_F^\star \nonumber
    \end{align}
\end{proposition}
\begin{proof}
    Consider a given time $k$ as in proposition \ref{prop:problem_feasible}.
    Since the optimistic path $\widetilde{P}_F$ is inside the set $(\mathcal{S}^E_k \cap \mathcal{S}_F^+)\cup (\mathcal{X}\mathcal{S}^E_k)$, the feasible space with future exploration is a subspace of the current feasible set.
    \begin{align}
        (\mathcal{S}^E_k \cap \mathcal{S}_F^+)\cup (\mathcal{X} \setminus \mathcal{S}^E_k) \subseteq (\mathcal{S}^E_{k+1} \cap \mathcal{S}_F^+)\cup (\mathcal{X} \setminus \mathcal{S}^E_{k+1})
    \end{align}
    Therefore, if no optimistic path $\widetilde{P}_F$ exists, also a feasible path $\Bar{P}_F$ cannot exist by further exploring the environment.
\end{proof}

Thus, exploring the unknown scene will either reveal a feasible path $\bar{P}$, or a detection that no feasible path exists. 
In the latter case, we can safely terminate without unnecessarily exploring the environment.%, as no feasible path exists.

% ---------------------------------------------------------------------------------------------------

\subsection{Optimality}
Once the feasibility of a solution has been established, we want to find the globally optimal follower path.
To this end, the cost of the current feasible path can be considered as the upper bound of the accumulated follower cost.

\begin{proposition}[Cost Upper Bound] \label{prop:upper_bound}
    The cost of the feasible path is an upper bound on the optimal path cost.
    \begin{align}
       Q_F(\Bar{P_F}, M) \geq Q_F(P_F^{\star}, C_F)
    \end{align}
\end{proposition}
\begin{proof}
    Since the feasible set $\mathcal{S}^E \cap \mathcal{S}_F^+$ is a subset of the full traversable set $\mathcal{S}_F^+$ and $C_F>0$, the cost of the feasible path is greater or equal to the optimal cost.
\end{proof}

Using the optimistic path, we aim to establish a lower bound on the follower path cost which we optimize as a surrogate objective.

\begin{proposition}[Cost Lower Bound] \label{prop:lower_bound}
    The cost of the optimistic path in the map completed with $c_F^{min}$ is the lower bound of the optimal path cost.
\end{proposition}
\begin{proof}
    Since $c_F^{min} \geq 0$, any increase in path length cannot decrease the cost. 
    Since the optimistic map assumes all unknown space is traversable and of minimal cost, exploration of more unknown space can never decrease the cost of an unknown cell or reduce the length of the optimistic path.
\end{proof}

As both the feasible and the optimistic paths can be computed based on the partial map, we now have a notion of the worst-case optimality gap.
Therefore, we show that if an optimistic path is feasible, the optimistic path and the feasible path have equal cost, and therefore the optimality gap is tight.

\begin{corollary}[Optimal Follower Path] \label{cor:termination_optimal}
    If the optimistic path $\widetilde{P}$ is feasible, it is the optimal path.
\end{corollary}

Ultimately, we can determine that the optimistic path is the optimal path when the condition in corollary~\ref{cor:termination_optimal} is satisfied. 
This allows us to terminate the exploration while guaranteeing that the optimistic path is the optimal path. 
In our experiments, we show that this termination criterion allows us to substantially reduce the necessary coverage required for the scout to cover. 

% --------------------------------------------------------------------------------------------------------

\section{Scouting \ac{IPP}}

In this section, we develop an \iac{IPP} method to plan the exploration of the scout. 
The core challenge is to reason about possible optimal follower paths based on only a limited partial cost field.
The key insight of our approach is to generate optimistic candidate paths to guide the scouting \ac{IPP}.
In particular, we propose a two-stage approach of i) prioritizing whether a feasible \emph{follower} path exists, and ii) discovering and establishing the optimal path. 
An overview of the proposed algorithm can be found in Alg.~\ref{alg:scout_ipp} and Fig.~\ref{fig:paper_overview}. 

\subsection{Informative Path Planning}
Our planner is based on a sampling-based \ac{IPP} approach from~\cite{Schmid20ActivePlanning}. 
The \ac{IPP} incrementally builds a tree by sampling candidate viewpoints. 
An \emph{information gain} $g$ is assigned to every node and a \emph{cost} $\gamma$ (denoted $c$ in \cite{Schmid20ActivePlanning}) to every edge in the tree.
The algorithm then computes a path as a sequence of nodes in the tree that maximizes the \emph{value} $v$, defined as accumulated gain divided by the accumulated cost for any trajectory.
To account for the new information acquired by the robot, only the first segment of that path is executed and the gains $g$, costs $\gamma$, and values $v$ are updated.
% The tree is then rewired to optimize $v$ for the new $g$ and $\gamma$.
% This incrementally builds an ever-growing tree until it covers the entire environment.
% The definition of the value $v$ thus yields an \emph{efficient} long-horizon path that is continuously updated based on the most recent robot information.

To optimize for exploration of a follower path, we make the following additions to the algorithm.
First, we define the information gain $g$ of a node $n$ as:

\begin{equation} \label{eq:information_gain}
    g(n) = \int_{\mathcal{S}^E_n } \mathbb{I}\left[\bm{x} \in \widetilde{P} \wedge \bm{x} \notin \mathcal{S}^E\right] d\bm{x}
\end{equation}
where $\mathcal{S}^E_n$ denotes the visible area from node $n$, and $\mathbb{I}$ is the indicator function which is $1$ for cells that are on the optimistic path $\widetilde{P}$ and are unexplored, and otherwise zero.
The cost $\gamma$ can conveniently be chosen as:
\begin{equation} \label{eq:ipp_cost}
    \gamma=\int_e C_S(\bm{x})d\bm{x}
\end{equation}
representing the scout travel cost \eqref{eq:problem_scout} along an edge $e$.

Importantly, a new optimistic path $\widetilde{P}$ is computed based on the current map $M$ during every update step, allowing accurate updates of $g$, $\gamma$, and $v$.
Observing that only unobserved cells in $M$ can contribute to the information gain, once all cells in view of a node $\mathcal{S}_n^E$ are explored, the node is added to a \emph{closed} set with $g=0$ and requires no further updates.

\subsection{Switching Exploration Modes}
To minimize scouting cost accumulated if the problem is infeasible, we first focus on discovering the feasibility of the problem and then optimality of the found path.

Following Prop.~\ref{prop:problem_feasible}, it is sufficient to find any feasible path.
Therefore, we complete the optimistic map $\widetilde{M}$ with the worst-case cost $c=c_F^{max}$.
This encourages $\widetilde{P}$ to stay within the already explored space $\mathcal{S}^E$ as much as possible, since traversing unknown areas incurs high cost.
Consequently, the scouting \ac{IPP} is encouraged to explore the geometrically shortest path connecting the currently explored and thus feasible parts of $\widetilde{P}$ to $\bm{x}_{goal}$.
%Before establishing feasibility, we use information gain in~\refequ{eq:information_gain} with the unknown cost as maximum \emph{follower} path cost $c=c^F_{max}$. As it is expensive to traverse the environment, this would guide the aerial scout to prioritize geometrically shortest paths, making the exploration to prioritize on discovering a feasible path. 

Once the feasibility of the problem has been established, meaning that we have found $\bar{P}_F$, we switch to exploring the \emph{optimal} path.
To this end, the optimistic map is completed with the minimal cost $c = c_F^{min}$.
As shown in Prop.~\ref{prop:lower_bound}, the resulting $\widetilde{P}$ reflects a lower bound on the optimal path cost.
Therefore, the scouting \ac{IPP} aims to observe as much of $\widetilde{P}$ as quickly as possible.
As a result, observing the lower bound follower path as a surrogate objective leads to increasingly higher cost lower bounds until the optimality gap is tight.

The scout \ac{IPP} terminates if either of the two conditions proposed in Cor.~\ref{cor:termination_optimal} and Prop.~\ref{prop:termination_infeasible} are satisfied. 
If Cor.~\ref{cor:termination_optimal} is satisfied, we have found the optimal \emph{follower} path and for Prop.~\ref{prop:termination_infeasible}, there exists no feasible path in the problem and therefore further exploring the environment does not make sense.

\begin{algorithm}[t]
\caption{Scouting \ac{IPP}}\label{alg:scout_ipp}
\begin{algorithmic}
\STATE $i\gets 0$
\STATE $\bm{x}_0 \gets \bm{x}_{start}$
\STATE $M_0 \gets M(\bm{x}_0)$
\STATE $c \gets c_F^{max}$
\WHILE{true}
    \STATE $\widetilde{M} \gets \text{complete}(M,c)$
    \STATE $\widetilde{P} \gets \hat{P}(\widetilde{M})$
    \IF{$!\widetilde{P}$}
        \STATE TERMINATE(infeasible) \COMMENT{Prop.\ref{prop:termination_infeasible}}
    \ENDIF
    \IF{$\bm{x} \in M \ \forall \bm{x} \in \widetilde{P}$}
        \IF{$c = c_F^{max}$}
            \STATE $c \gets c_F^{min}$ \COMMENT{Switch exploration mode}
        \ELSE
            \STATE TERMINATE(optimal) \COMMENT{Cor.\ref{cor:termination_optimal}}
        \ENDIF
    \ENDIF
    \STATE $\bm{u} \gets \argmax{n} \ v(n)$ \COMMENT{Sampling-based IPP}
    \STATE $\bm{x}_{i+1} \gets f(\bm{x}_i, \bm{u})$
    \STATE $M = M \bigcup M(x_{i+1})$
    \STATE $i = i + 1$
\ENDWHILE
\end{algorithmic}
\end{algorithm}

\subsection{Completeness}
Lastly, combining the results from \refsec{sec:follower_path_planning} with the scout \ac{IPP} proposed in this section, we show that the proposed approach is complete, meaning that the optimal path will eventually be found if one exists.

\begin{proposition}[Scout \ac{IPP} Completeness] \label{prop:path_completeness}
    The scout will find the optimal \emph{follower} path if it exists.
\end{proposition}
\begin{proof}
    The scouting \ac{IPP} only adds nodes $n$ to the tree at every iteration. 
    Since the number of map cells is finite, the tree will eventually cover all reachable cells in the map.
    As long as there exists an unknown cell in the optimistic path of the current map, 
    since only nodes $n$ that observe unknown cells on the optimistic path have a gain $g>0$, the scouting \ac{IPP} will explore these cells.
    Therefore, the optimistic path must eventually become feasible, and thus the optimal path is found Cor.~\ref{cor:termination_optimal}.
    Since $\mathcal{S}^+_F \subseteq \mathcal{S}^+_S$, the optimal follower past must be reachable by the scout or infeasibility of the problem can be established Prop.~\ref{prop:termination_infeasible}
\end{proof}

While Prop.~\ref{prop:path_completeness} implies that in the worst case, the map must be densely sampled and exhaustively explored to guarantee optimality, we note that in practice optimality is typically shown much earlier and sparse search trees suffice.

% ===============================================================================================
\section{Experiments}
\label{sec:evaluation}

\subsection{Experimental Setup}
\label{sec:exp_setup}

We evaluate our approach in a planetary exploration scenario, where an MAV scout is deployed to explore an unknown environment in order to find the lowest cost path for a ground rover. 
An overview of all components is shown in Fig.~\ref{fig:system_overview}.
Without loss of generality, we assume that the aerial scout and ground rover are controlled over a two-dimensional space $\mathcal{X}\subset\mathbb{R}^2$ and represent the environment as a fixed-resolution grid map $M \in \mathcal{C}^{640 \times 480}$, using $\SI{0.5}{m}$ resolution. 
For teamed exploration, we assume that the \emph{scout} and \emph{follower} start from the same position $\bm{x}_{start}$.

\parsec{Environment}
We use OAISYS \cite{muller2022interactive, dlr145997} to simulate high-fidelity planetary environments.
The environment used spans an area of $\SI{320}{}\times\SI{240}{m}$, consisting of three different types of soils, corresponding to different rover traversal costs $C_F$. Additionally, the environment contains obstacles of various sizes.
Scenes are inspired by real Mars environments in terms of obstacle distribution and number of soil categories~\cite{krautClassificationMarsTerrain2010, swanAI4MARSDatasetTerrainAware2021}.
Example scenes are shown in Fig.~\ref{fig:maps}.
We use a constant cost for the scout $C_S(\bm{x}) = 1$, thus minimizing exploration time.

\parsec{Robot Parameters}
We use a 2D constant acceleration motion model with a maximum velocity of $v_{max}=\SI{10}{m/s}$.
The aerial scout is equipped with a downward-facing camera with a 90$^\circ$ field of view (FoV) and at a constant flying height of $\SI{20}{m}$, thus covering $\approx \SI{40}{} \times \SI{40}{m}$.
Since the main focus of our work is on planning, we assume that accurate state estimates and soil-type measurements are provided by the camera.

\parsec{Metrics}
Since the planner is stochastic, every experiment is conducted 10 times and the mean and standard deviations are reported.
To assess the performance of the planner, we compute the average time \tfirst until a \emph{feasible} path $\bar{P}_F$ is found, average time \topt until the \emph{optimal} path $P_F^\star$ is found, and the average time \tterm until exploration is terminated.
Note that often the optimal path is discovered before the scout can determine that it is optimal, thus \topt$\leq$ \tterm.

\parsec{Baselines}
Since, to the best of our knowledge, there are no methods that address our problem exactly, we compare our \emph{Path-aware} planner against the most closely related baselines.
We compare against a pure \emph{Exploration} planner \cite{Schmid20ActivePlanning}, and a \emph{Goal-aware} planner \cite{oleynikovaSafeLocalExploration2018} aiming to explore a path towards a known goal point as fast as possible.
To also consider the observed follower cost, we take inspiration from the ideas of \cite{delmericoActiveAutonomousAerial2017, zhangFastActiveAerial2022} and create two additional baselines for our problem.
A \emph{Frontier-cost} planner computes frontiers, where each frontier cell has an information gain inverse to the cost of the adjacent observed space.
In addition, a \emph{Cost-aware} planner inversely weighs the information gain of each view by the average cost of the known voxels in the view, thus encouraging broader exploration adjacent to low-cost areas.

\parsec{Hardware} All experiments are conducted on an Intel i7-9750H laptop CPU @\SI{2.6}{\giga\hertz}, which can be equipped on mobile aerial robots. We note that a maximum of 2 cores were used by our planner. 

\begin{figure}
    \centering
    \includegraphics[width=\linewidth]{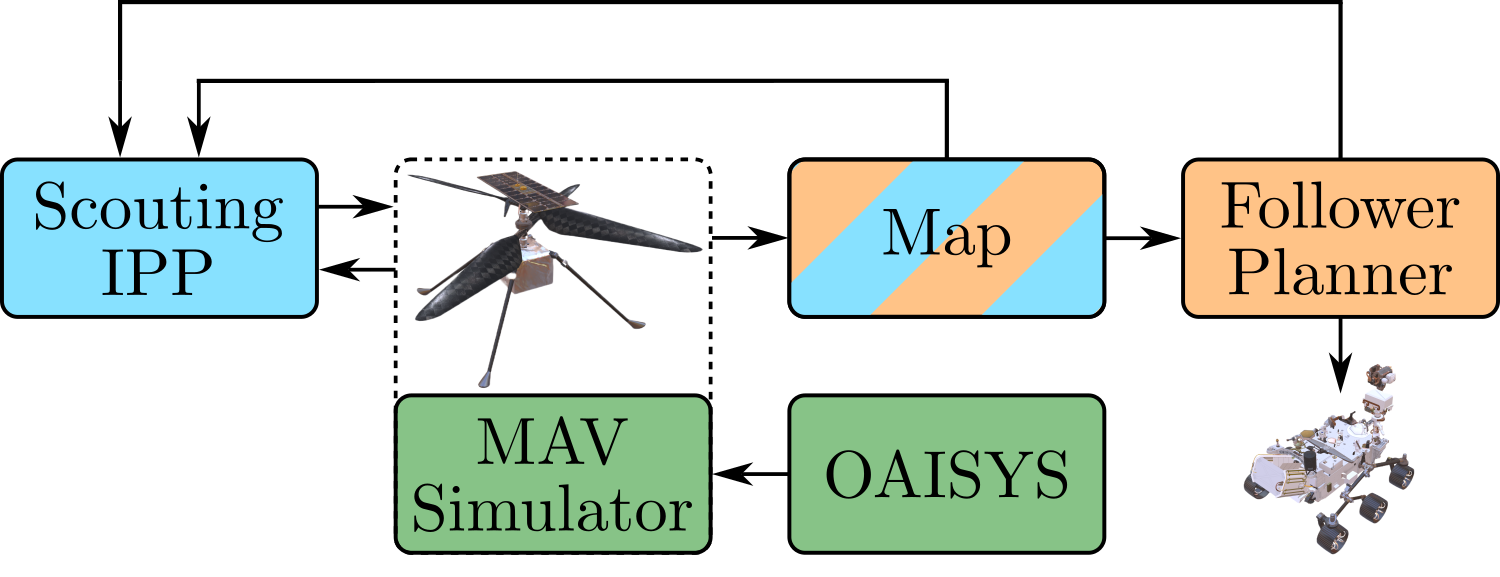}
    \caption{Overview of the proposed approach. The scout adds observations to the map. The follower planner computes feasible and optimistic paths on the given map. These paths are used to guide future exploration by the scout. Once the optimal follower path has been established, the scout terminates and the path can be executed. (Images: NASA/JPL)}
    \label{fig:system_overview}
    \vspace{-10pt}
\end{figure}

\begin{figure}
\centering
    \includegraphics[width=0.325\columnwidth]{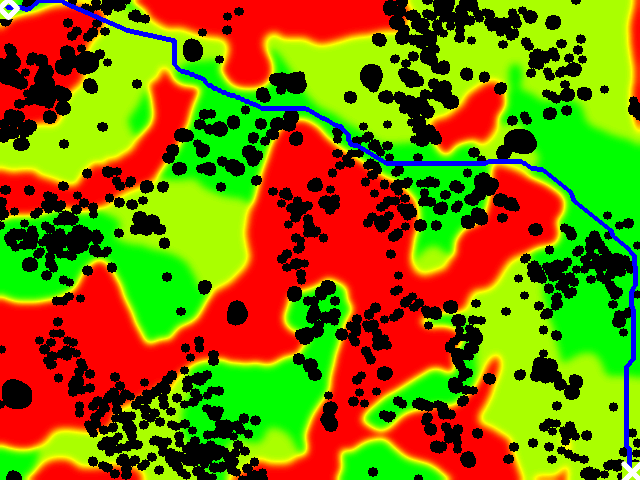}
    \hfill
    \includegraphics[width=0.325\columnwidth]{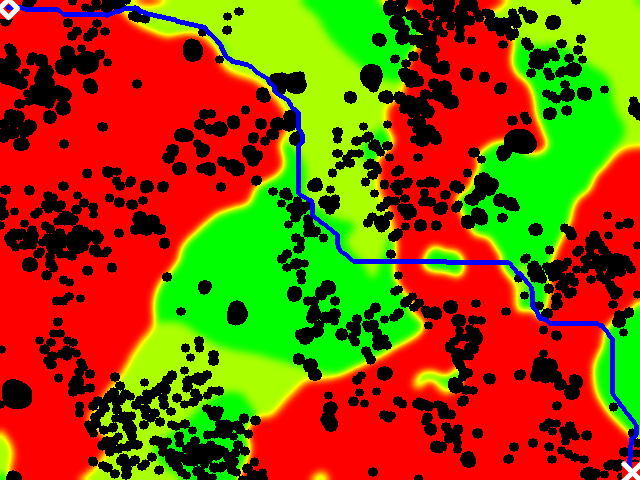}
    \hfill
    \includegraphics[width=0.325\columnwidth]{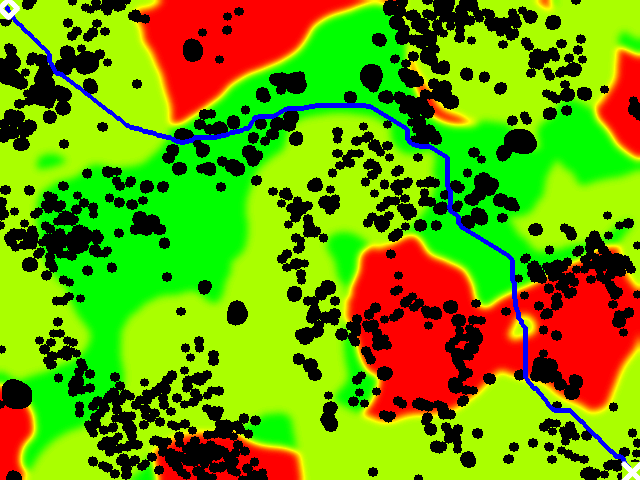}
  \caption{Different environments used for evaluation. Variable follower cost fields are shown in color from red (high) to green (low). The optimal collision-free path is shown in blue. }\label{fig:maps}
      \vspace{-15pt}
\end{figure}

\begin{figure*}
    \centering
  \begin{subfigure}[b]{0.32\linewidth}
    \includegraphics[width=\linewidth]{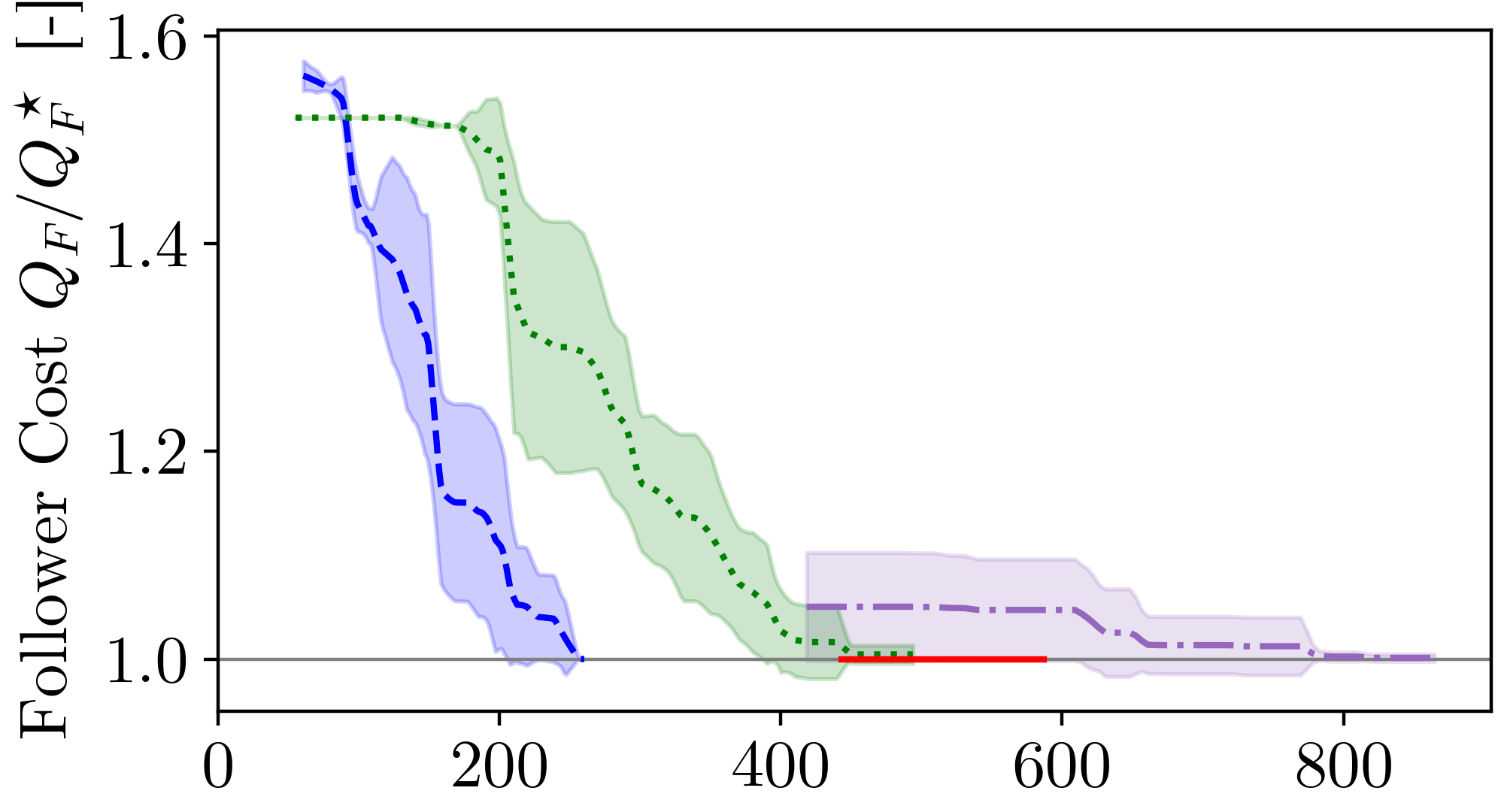}
    \label{fig:cost_112}
     \vspace*{-1.2em}
  \end{subfigure}
  \hfill
  \begin{subfigure}[b]{0.32\linewidth}
    \includegraphics[width=\linewidth]{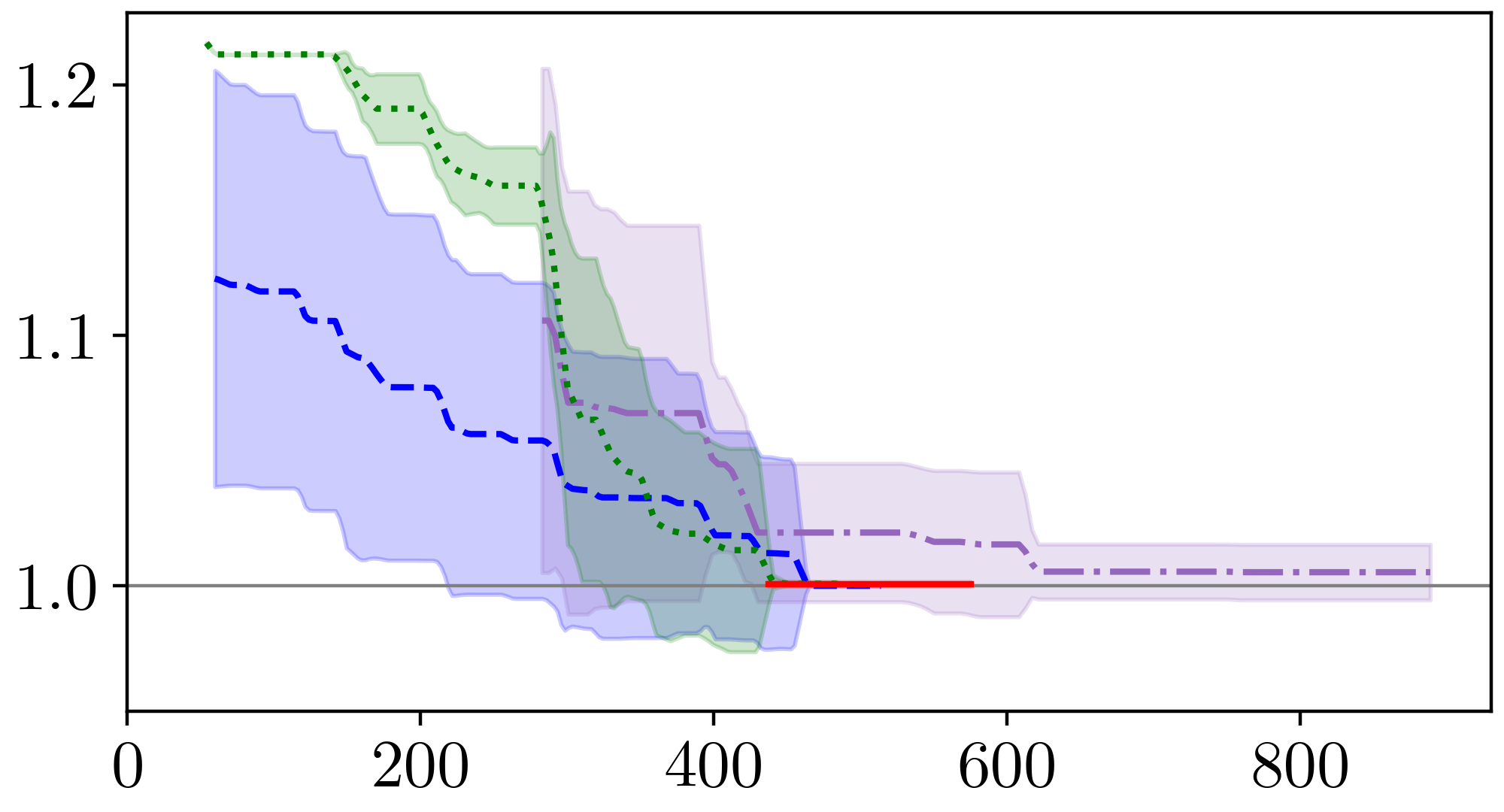}
    \label{fig:cost_212}
    \vspace*{-1.2em}
  \end{subfigure}
  \hfill
  \begin{subfigure}[b]{0.32\linewidth}
    \includegraphics[width=\linewidth]{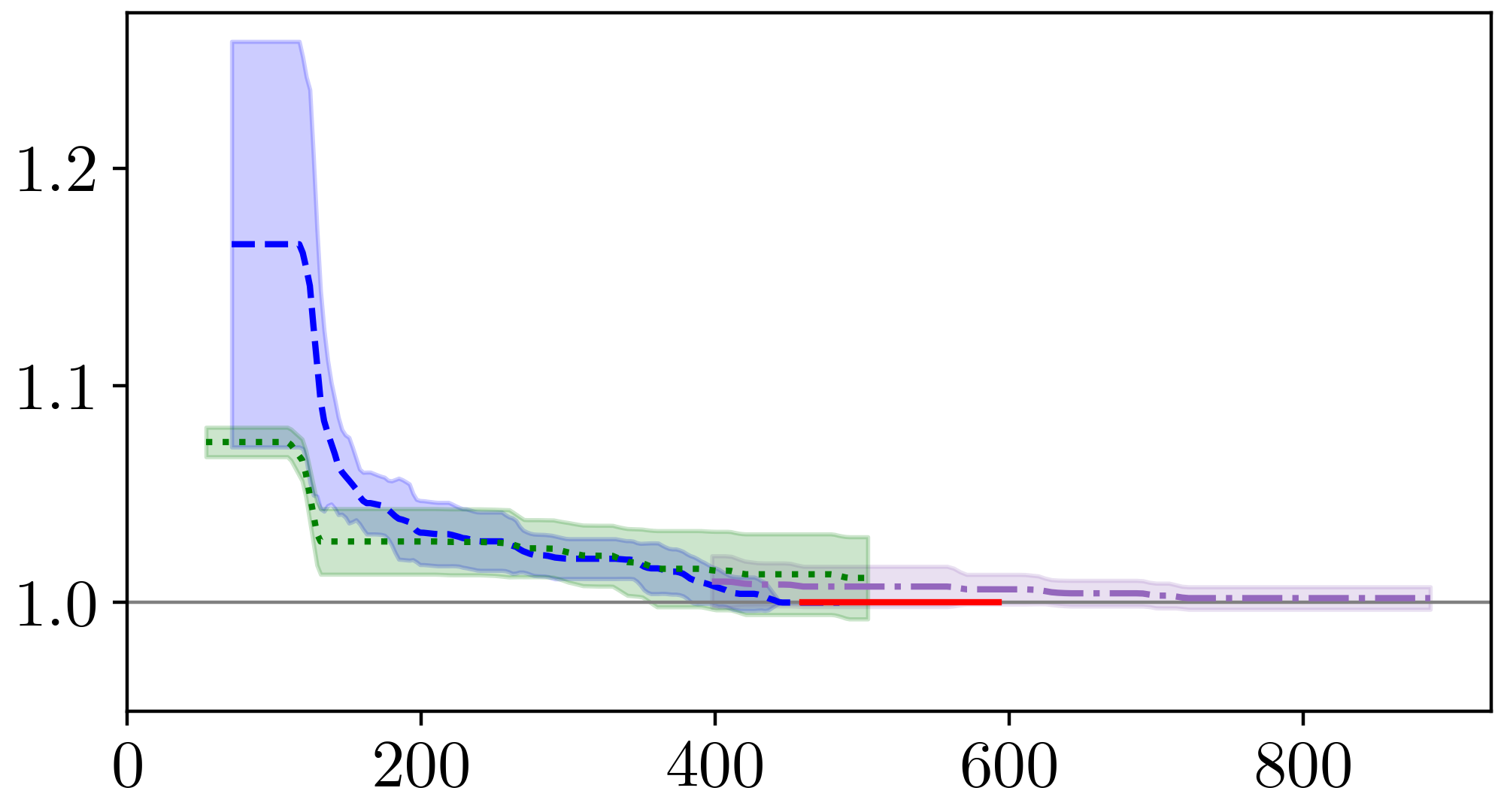}
    \label{fig:cost_312}
    \vspace*{-1.2em}
  \end{subfigure}
  \begin{subfigure}[b]{0.32\linewidth}
    \includegraphics[width=\linewidth]{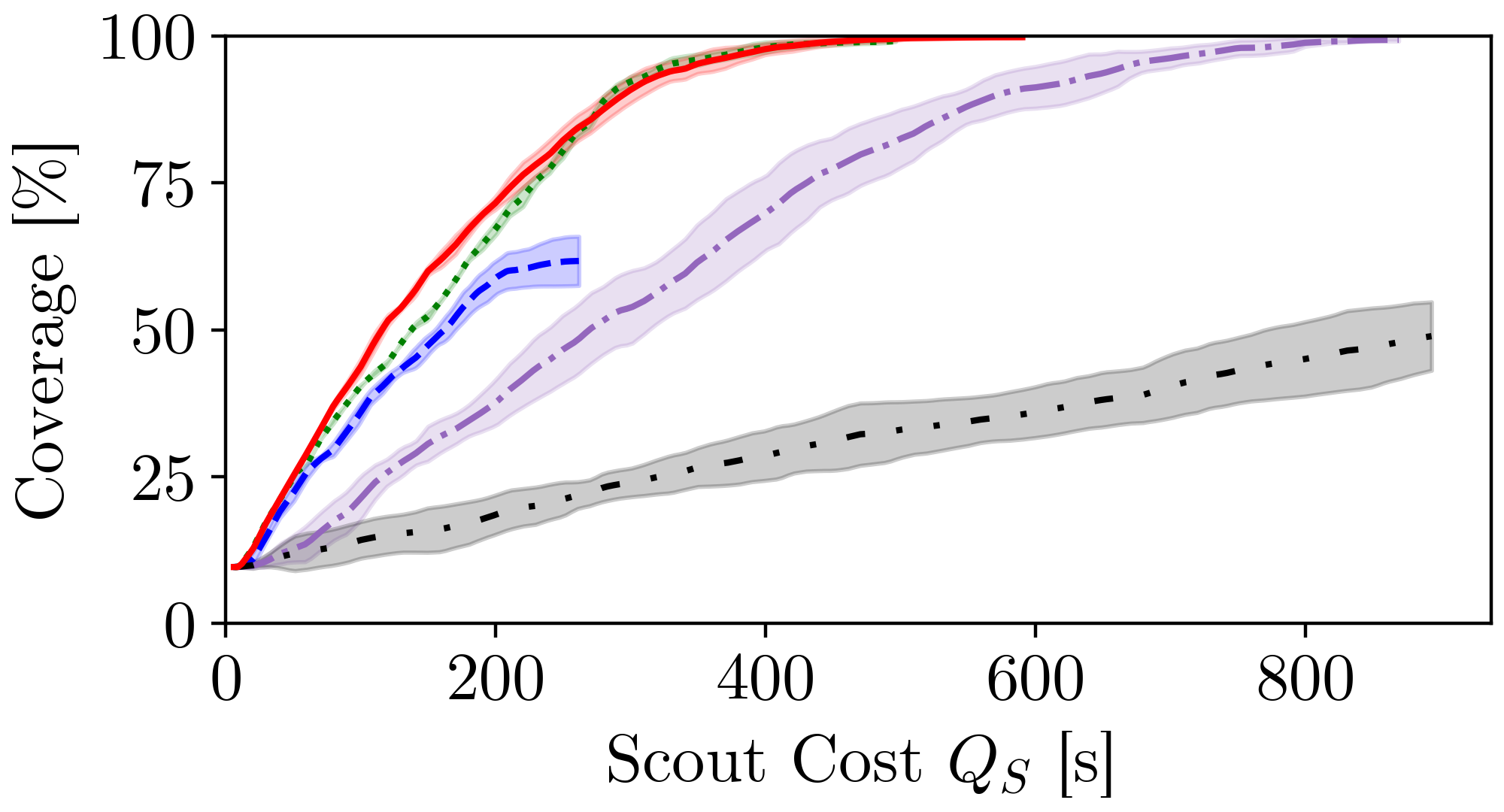}
    \label{fig:coverage_112}
    \vspace*{-1.0em}
  \end{subfigure}
  \hfill
  \begin{subfigure}[b]{0.32\linewidth}
    \includegraphics[width=\linewidth]{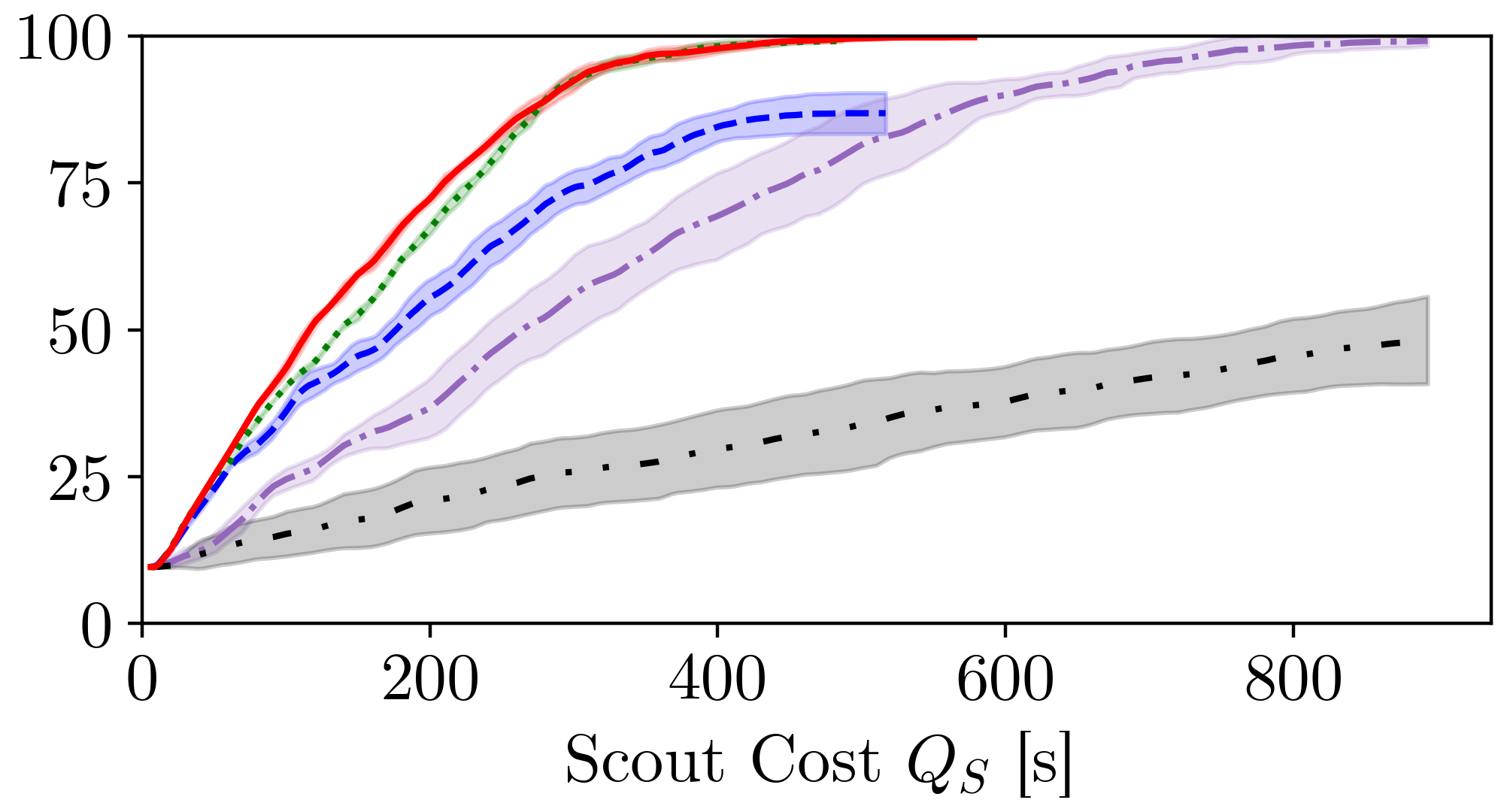}
    \label{fig:coverage_212}
    \vspace*{-1.0em}
  \end{subfigure}
  \hfill
  \begin{subfigure}[b]{0.32\linewidth}
    \includegraphics[width=\linewidth]{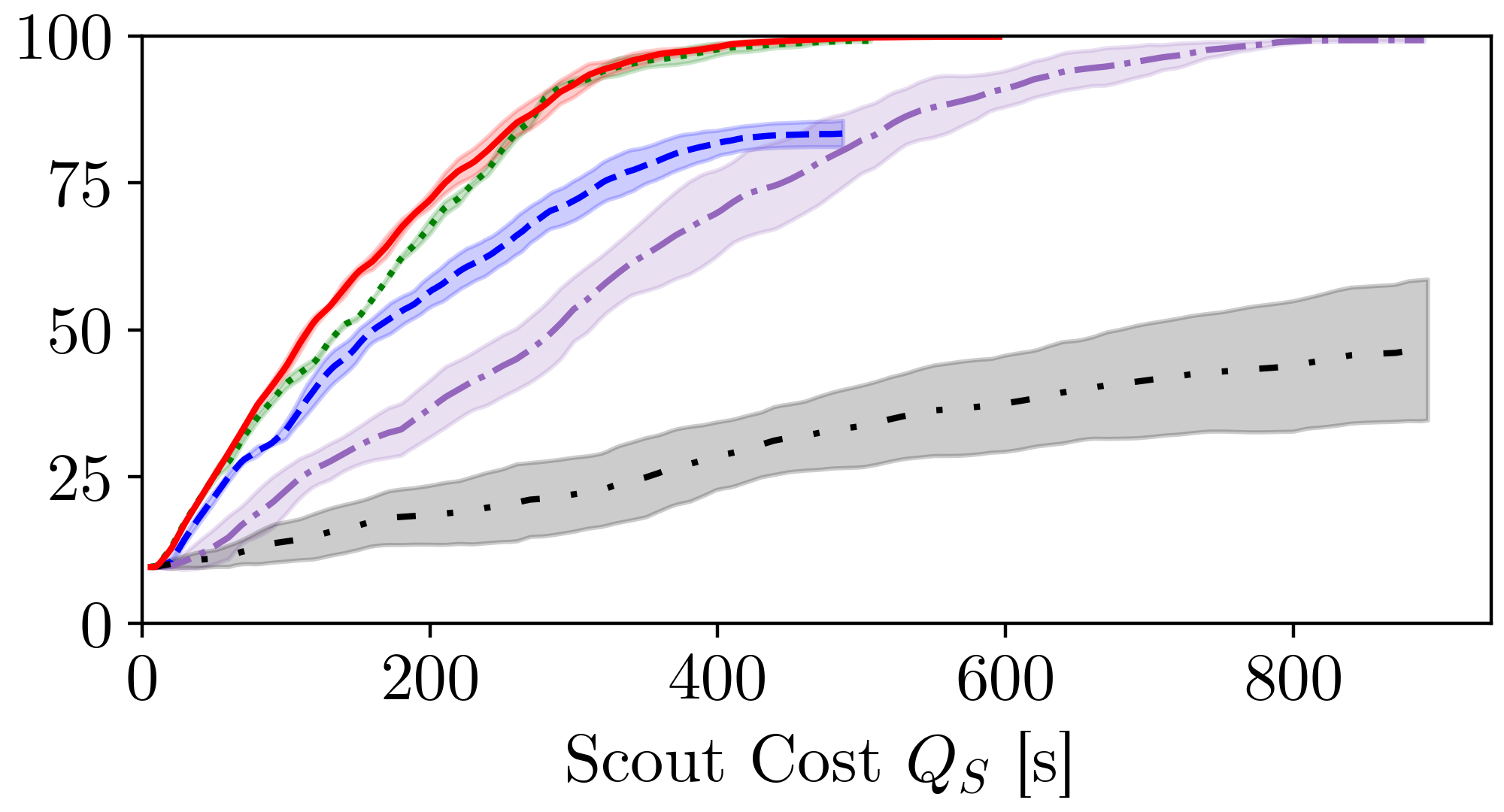}
    \label{fig:coverage_312}
    \vspace*{-1.0em}
  \end{subfigure}
    \begin{subfigure}[b]{0.7\linewidth}
    \includegraphics[width=\linewidth]{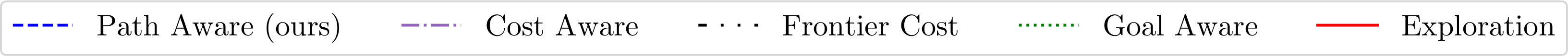}
    \label{fig:legend}
    \vspace*{-1.2em}
  \end{subfigure}
    \caption{Normalized follower path cost $Q_F(t)/Q_F^\star$ and scene coverage after $Q_S(t) = t$ seconds of scouting for different methods. }
    \label{fig:general_comparison}
    \vspace{-10pt}
\end{figure*}

% -----------------------------------------------------------------------------------------------
\subsection{Scouting for Optimal Follower Paths}
\label{sec:evaluation_comparison}

We first evaluate the ability of each method to efficiently scout for optimal follower paths in three different scenes. 
Fig.~\ref{fig:general_comparison} shows the current follower path cost $Q_F(t)/Q_F^\star$ after $Q_S(t) = t$ seconds of scouting.
Note that the follower cost $Q_F(t)$ is normalized to the optimal path $Q_F^\star$ in that environment, thus a value of $1$ indicates the optimal path has been found.
We further show the \emph{coverage}, \ie the percentage of the scene that was explored.
Each column (left to right) corresponds to an environment from Fig.~\ref{fig:maps}.
Finally, the performance metrics are summarized in Tab.~\ref{tab:metrics}

We observe in all scenarios that our \emph{path-aware} and the \emph{goal-aware} methods find a feasible path the earliest, as also indicated in the notably lower \tfirst.
Thus, in particular when the scout travel budget is low, knowledge of the goal location is essential to identify an initial path quickly.
However, once a feasible path is established, our approach optimizing the lower bound of $Q_F$ is able to drive down $Q_F$ significantly faster than the goal-aware planner, indicated by a $22.5\%$ decrease in \topt on average.

We note that for our path-aware planner, once the optimal path is found, only little more time needs to be spent before the planner can establish that the opimality gap is tight and the path is indeed optimal.
This allows our approach to terminate $17\%$ quicker than the fastest coverage-based approach.
Interestingly, we find that it is sufficient for our method to explore only $77.0\%$ of the scene on average in order to establish optimality, saving valuable scout flight time. 
It is worth pointing out that the map space is allocated to fit the start and goal positions. 
The required coverage could thus be markedly lower if even larger spaces are considered.

As expected, the pure \emph{exploration} approach achieves the quickest coverage rate.
However, without being aware of the follower robot plan, the goal location is often only explored at a later stage.
On the other hand, the scene is typically sufficiently explored for the first path to be already close to the optimal one, as indicated by a flat line between $\sim\SI{400}{}-\SI{600}{s}$ in Fig.~\ref{fig:general_comparison}.
Nonetheless, the remainder of the scene has to be explored to establish optimality.
These observations highlight that the quickest coverage rate is not necessarily the best strategy, as our approach with slower scene coverage but focusing on relevant areas achieves better performance.

Interestingly, the two cost-aware baselines don't perform very well.
Similar to ours, incorporating the cost for view planning naturally slows down overall exploration progress.
However, in the \emph{cost-aware} and \emph{frontier-cost} planners it does not lead to a quicker exploration of a feasible or the optimal path, such that cost-unaware but rapid exploration appears to be a superior strategy.
This highlights the challenging nature of our problem and that it is not trivial to incorporate traversal cost information into a successful scout-follower-IPP system.
In contrast, optimizing the lower bound of the cost over the complete follower trajectory appears to be a more appropriate information signal, leading to improved performance over cost-unaware methods.

\setlength{\tabcolsep}{2pt} % Default value: 6pt
\begin{table}
\centering
\resizebox{\columnwidth}{!}{%
\input{include/mean_std_table.tex}
}
\caption{Performance metrics for different methods [s].}
\label{tab:metrics}
\vspace{-5pt}
\end{table}

\begin{figure}
    \centering
    \includegraphics[width=0.95\linewidth]{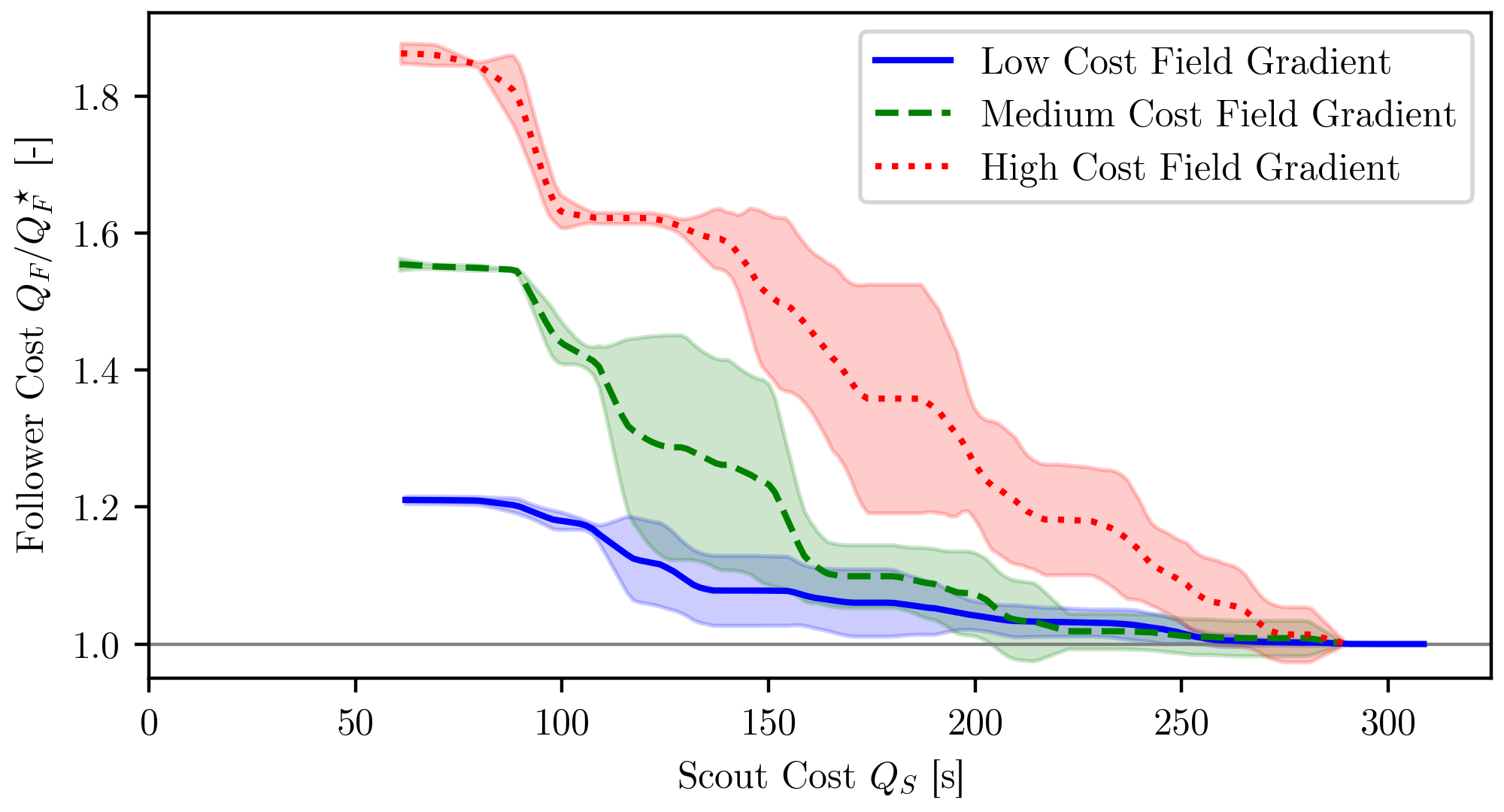}
    \caption{Scouting performance for varying cost gradients in a scene.}
    \label{fig:cost_comp}
    \vspace{-10pt}
\end{figure}

% -----------------------------------------------------------------------------------------------

\subsection{Influence of the Follower Cost Distribution}
\label{sec:evaluation_cost}

Since follower cost can vary notably for different environments and robots, we evaluate our approach for an identical distribution of soil types, but with different \emph{cost gradients}.
We use an exponential set of costs, where $C_F^{max}=2\times$, $4\times$, and $8\times C_F^{min}$ for \emph{low}, \emph{medium}, and \emph{high} cost gradients, respectively.
Intuitively, the gradient expresses how "bad" undesirable regions of the scene are, where a high gradient encourages longer paths that avoid bad areas over more direct paths through them.
We observe in Fig.~\ref{fig:cost_comp} that for higher gradients the initially found paths are naturally more sub-optimal.
Nonetheless, our approach is able to refine the path quickly in all situations.
Notably, the optimal path is discovered in very similar time, and optimality is established and planning terminated shortly after.
This suggests that the presented formalism is robust with respect to different cost gradients.

% -----------------------------------------------------------------------------------------------

\subsection{Influence of Obstacle Density}\label{sec:evaluation_obstacles}

We further study the importance follower obstacles for successful scouting in Fig.~\ref{fig:obst_comp} by considering an increasing amount of obstacles, up to $32.6\%$ of the scene, illustrated in Fig.~\ref{fig:obstacle_densities}.
We find that unless there are many obstacles, the impact on planning performance is minimal.
In the high density case, the optimal path is mainly determined by traversability rather than cost, leading to an increase of $\bar{\tau}_1$ from 65s for medium density to 111s. 
The same holds also for $\bar{\tau}^*$. 
The relative initial follower cost on the other hand is lower for the high density case, since the difference between most feasible paths and the optimal path is lower. 

% -----------------------------------------------------------------------------------------------

\begin{figure}
    \centering\includegraphics[width=0.95\linewidth]{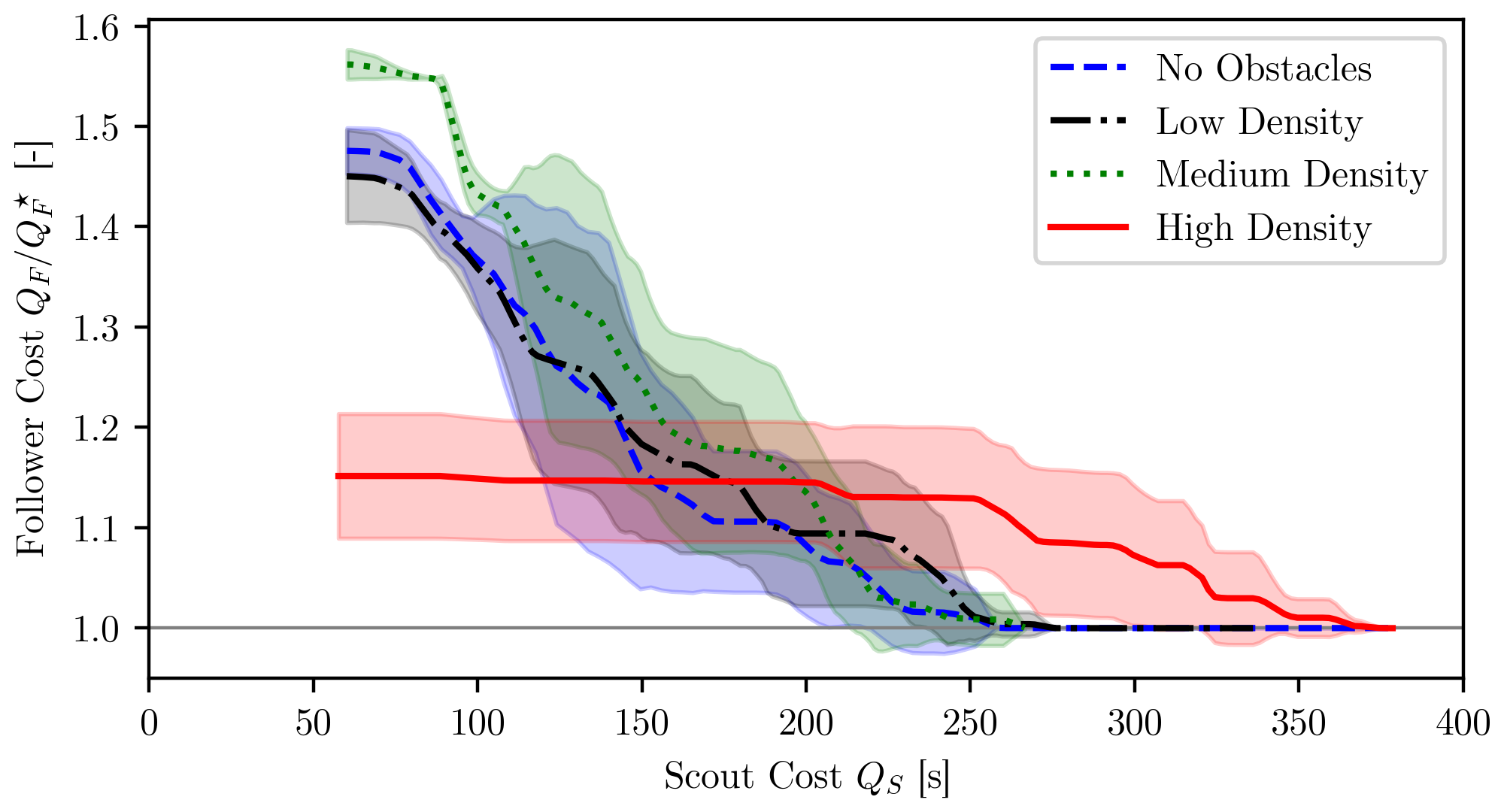}
    \caption{Scouting performance for variable obstacle densities.}
    \label{fig:obst_comp}
    \vspace{-5pt}
\end{figure}

\begin{figure}
    \includegraphics[width=0.24\columnwidth]{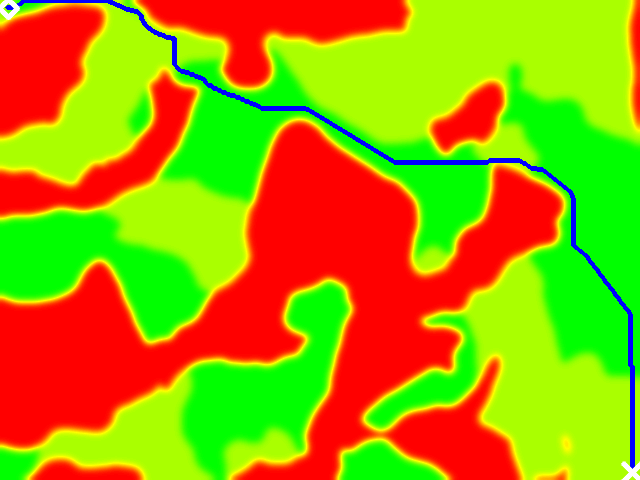}
    \includegraphics[width=0.24\columnwidth]{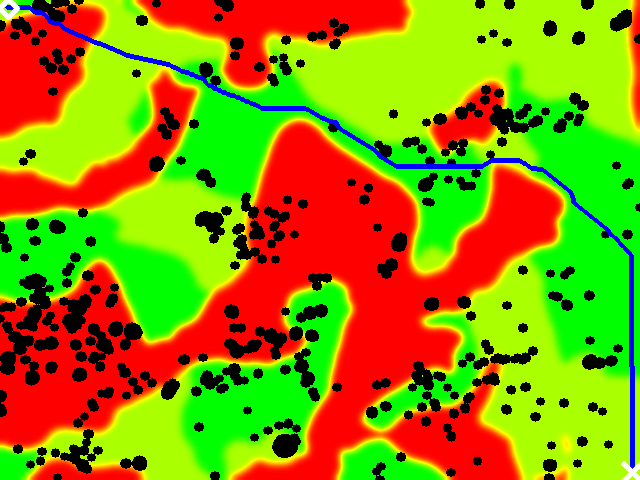}
    \hfill
    \includegraphics[width=0.24\columnwidth]{images/map_112_vis.png}
    \hfill
    \includegraphics[width=0.24\columnwidth]{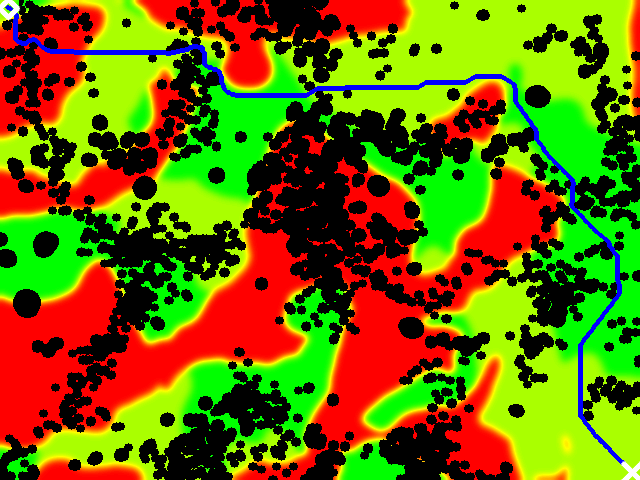}
    \caption{Environment with varying obstacle distributions and densities, covering from 0\% (left) to 32.6\% (right) of the scene.}
    \label{fig:obstacle_densities}
    \vspace{-10pt}
\end{figure}

\subsection{Generalizing to Different Follower Planners}
\label{sec:evaluation_prm-rrt}

Our proposed IPP method is agnostic to the follower planner formulation, assuming it is complete and optimal as described in \refsec{sec:follower_path_planning}. 
We evaluate the impact of relaxing this assumption by comparing the performance of scouting for our deterministic $A^*$ planner and the popular sampling-based $RRT^*$ planner \cite{karamanSamplingbasedAlgorithmsOptimal2011}, that is only \emph{asymptotically} complete and optimal.
The results in Fig.~\ref{fig:prm_vs_rrt} demonstrate that, although optimality and completeness of our scouting IPP can not be strictly guaranteed, very similar performance is achieved in practice.
We note that the variance for $RRT^*$ is slightly higher, which can be explained by the sampling-based nature of the planner. 
Nonetheless, the fact that \tfirst, \topt and \tterm are almost identical suggests that our method can be paired with a broad range of follower planners.

\subsection{Analysis of Stopping Criteria}
\label{sec:evaluation_qualitative}
To illustrate the ability of our approach to establish the globally optimal path even in challenging environments and terminating early if no feasible path exists, we consider the synthetic environments shown in Fig.~\ref{fig:artificial_maps}. 
We observe that for the open box scene, our approach only needs to explore$\approx16\%$ of the free space to guarantee that the path is optimal. 
Similarly, in the closed box case our planner explores only $18\%$ of the scene, allowing it to establish that the problem is infeasible and terminating early.

% \begin{table}
%     \centering
%     \begin{tabular}{c|c|c}
%      $\bar{\tau}_1$&  $\bar{\tau}^*$&$\bar{\tau}_\infty$\\\hline
%      $121\sec$& $155\sec$ & $205\sec$
% \end{tabular}
%     \caption{Average performance metrics for follower path-aware scouting IPP in open-box map over 10 runs.}
%     \label{tab:map_815_metrics}
% \end{table}
% closed: $\bar{\tau}_\infty$ of $109\sec$,

\begin{figure}
    \centering
    \includegraphics[width=0.95\linewidth]{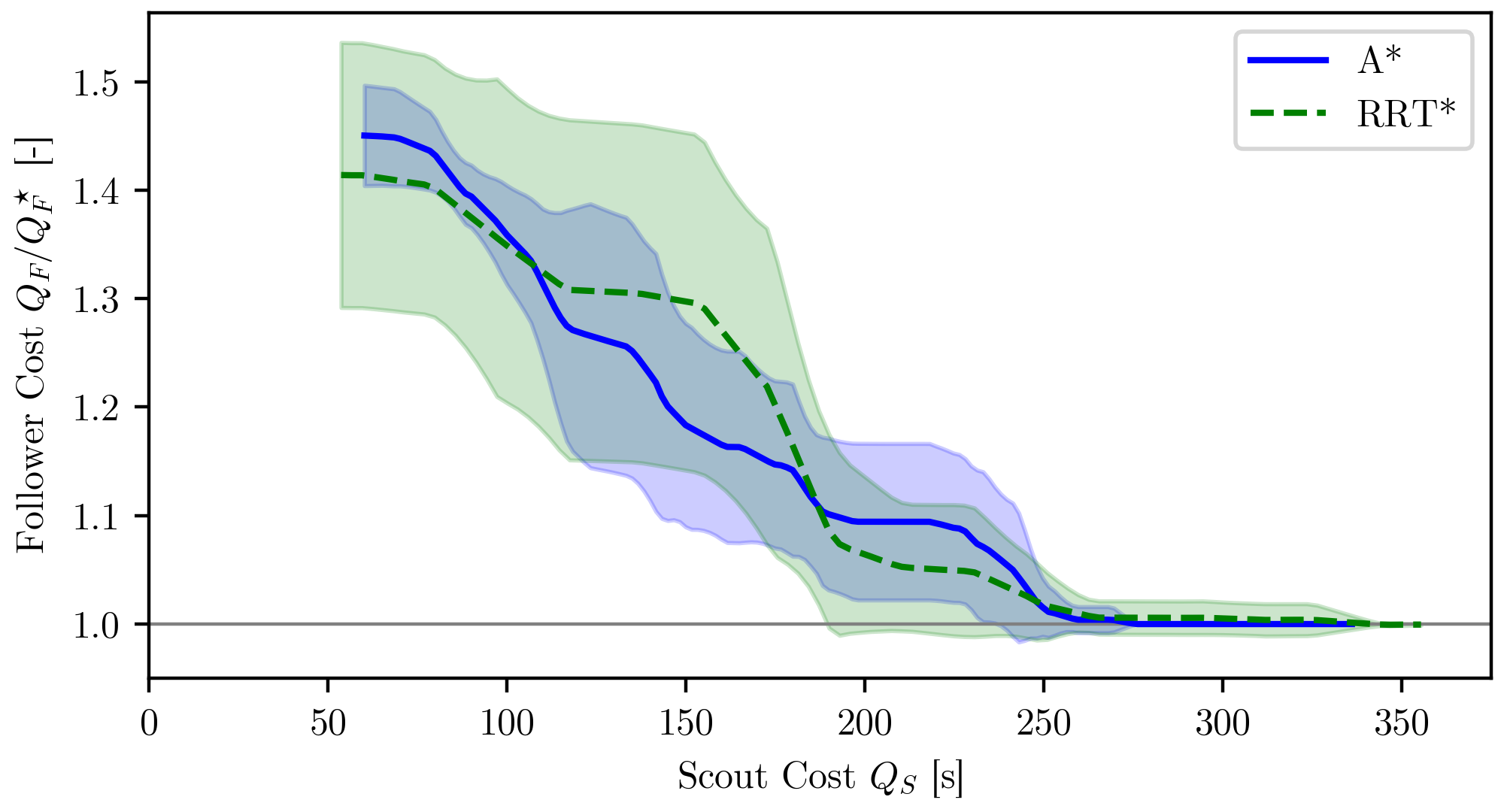}
    \caption{Scouting performance for a deterministic $A^*$ compared to a probabilistic $RRT^*$ follower planner.}
    \label{fig:prm_vs_rrt}
    \vspace{-5pt}
\end{figure}

\begin{figure}
\centering
    \includegraphics[width=0.49\linewidth]{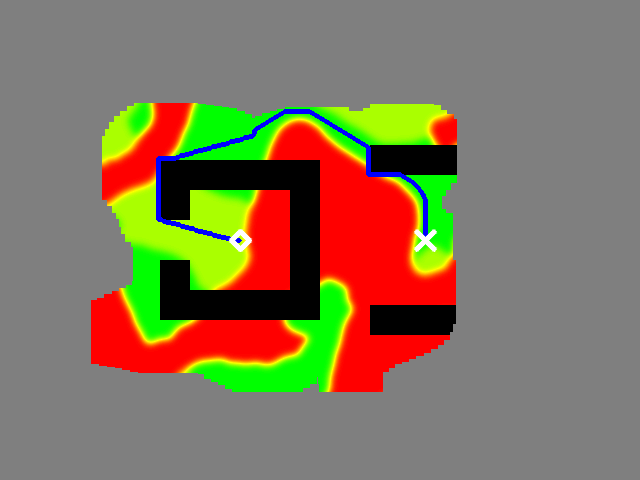}
    \includegraphics[width=0.49\linewidth]{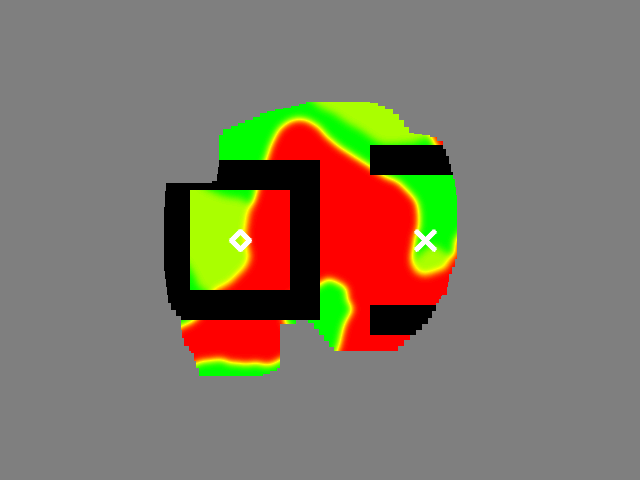}
  \caption{Explored space until termination for the open (left, optimal path found) and closed box (right, problem infeasible) environments. The start and goal is marked as a rhombus, and a cross, respectively.}
  \label{fig:artificial_maps}
  \vspace{-10pt}
\end{figure}

% ===============================================================================================

\section{Conclusions}
\label{sec:conclusion}

In this work, we present a novel \ac{IPP} formulation for a collaborative robotic team consisting of an aerial scout and a follower robot in unknown scenes. 
% We propose an approach that optimizes the lower bound of the follower trajectory cost as a surrogate objective, allowing efficient exploration of the optimal follower path.
We derive theoretical guarantees for our method find an optimal path for the follower robot on termination.
% We achieve this by introducing two termination criteria when the problem is infeasible or the optimal path is established.
We show in thorough experimental evaluation that our approach outperforms existing methods and the derived guarantees hold in practice, leading to a significant decrease in time to explore the optimal path as well as time to termination. 
% We demonstrate that our proposed approach generalizes well to a wide range of environments and shows strong performance in practice even when follower planner assumptions are relaxed.
% We release our implementation open-source.

However, our current approach assumes perfect sensing and state estimation to derive rigorous guarantees. 
Introducing uncertainty and extending the proofs to probabilistic guarantees would better reflect real-world operations and increase robustness against \eg misclassification of the terrain.
We also observed that if there are two similar-cost homotopy classes of the follower path, the optimistic path alternates between them, leading to less effective exploration. 
This could be addressed using a momentum-based or multi-modal planner in the future.
% Lastly, planning for a larger set of potentially heterogeneous agents is an interesting problem.

% We expect this work to provide a formal relationship when the task of exploration(scout) and exploitation(follower) is assigned to different agents in the collaborative robot team.===============================================================================================

% \section*{ACKNOWLEDGMENT}

% ===============================================================================================

{\small
\bibliographystyle{IEEEtran}
% \balance
\bibliography{IEEEfull,references}
}

%%%%%%%%%%%%%%%%%%%%%%%%%%%%%%%%%%%%%%%%%%%%%%%%%%%%%%%%%%%%%%%%%%%%%%%%%%%%%%%%

\end{document}

%% file: include/acronyms.tex
\DeclareAcronym{SLAM}{
  short = SLAM,
  long  = simultaneous localization and mapping,
  short-indefinite = a,
  long-indefinite = a
}

\DeclareAcronym{NBV}{
  short = NBV,
  long  = next-best-view,
  short-indefinite = a,
  long-indefinite = a
}

\DeclareAcronym{IPP}{
  short = IPP,
  long  = informative path planning,
  short-indefinite = an,
  long-indefinite = an
}

\DeclareAcronym{MAV}{
  short = MAV,
  long  = micro aerial vehicle,
  short-indefinite = a,
  long-indefinite = a
}

%% file: include/newcommands.tex
\newcommand{\refsec}[1]{Section~\ref{#1}}

%% file: include/mean_std_table.tex
\begin{tabular}{c|ccccc}
\toprule
Metric & Path Aware (ours) &     Cost Aware &  Frontier Cost &    Goal Aware &   Exploration \\
\midrule
\tfirst [s]     &       $81 \pm 24$ &  $597 \pm 135$ &  N/A &    $54 \pm 0$ &  $514 \pm 38$ \\
\topt [s] &      $293 \pm 80$ &  $656 \pm 101$ &  N/A &  $378 \pm 41$ &  $514 \pm 38$ \\
\tterm [s]   &      $376 \pm 30$ &   $785 \pm 56$ &  N/A &  $454 \pm 35$ &  $522 \pm 46$ \\
\bottomrule
\end{tabular}
%\begin{tabular}{|c|ccccc|}
%    \hline
%    Metric & Exploration & Goal-aware & Cost-aware & Frontier-cost & Path-aware (ours) \\
%    \hline 
%    \tfirst & 123 $\pm$ 4 &\\
%    \topt & \\
%    \tterm & \\
%    \hline
%\end{tabular}